\documentclass{article}

\usepackage[preprint]{neurips_2025}

\usepackage[utf8]{inputenc}
\usepackage[T1]{fontenc}
\usepackage{amsmath,amssymb,amsthm}
\usepackage{algorithm}
\usepackage{algorithmic}
\usepackage{booktabs}
\usepackage{graphicx}
\usepackage{subcaption}
\usepackage{hyperref}
\usepackage{url}
\usepackage{xcolor}
\usepackage{microtype}
\usepackage{bm}

\newcommand{\R}{\mathbb{R}}

\newcommand{\KL}{\mathrm{KL}}
\newcommand{\LSE}{\mathrm{LSE}}
\newcommand{\eps}{\varepsilon}
\newcommand{\bpi}{\bm{\pi}}
\newcommand{\balpha}{\bm{\alpha}}
\newcommand{\bbeta}{\bm{\beta}}
\newcommand{\bmu}{\bm{\mu}}
\newcommand{\bnu}{\bm{\nu}}

\newtheorem{theorem}{Theorem}

\title{Fast Log-Domain Sinkhorn Optimal Transport\\with Warp-Level GPU Reductions}

\author{
  Hao Xiao \\
  ATLAS AI Lab \\
  \texttt{xiaohao@atlasthinktank.com}
}

\begin{document}

\maketitle


\begin{abstract}
Entropic regularized optimal transport (OT) via the Sinkhorn algorithm has become a fundamental tool in machine learning, yet existing implementations either suffer from numerical instability for small regularization parameters or incur significant overhead from deep learning frameworks. We present \textsc{FastSinkhorn}, a lightweight, native CUDA implementation of the log-domain Sinkhorn algorithm that combines warp-level shuffle reductions with shared-memory tiling to achieve high GPU utilization without sacrificing numerical stability. Our solver operates entirely in the log-domain, enabling robust computation for regularization parameters as small as $\varepsilon = 10^{-4}$ where standard-domain methods fail. On dense OT problems with $n = m = 8192$, our implementation achieves $12\times$ speedup over the widely-used POT library and $5.9\times$ speedup over GPU-accelerated PyTorch baselines, while consuming only 256\,MB of GPU memory. We validate our solver on image color transfer, 3D point cloud matching, and convergence analysis, demonstrating that native CUDA kernels with careful numerical treatment provide a practical and efficient foundation for large-scale optimal transport computation.
\end{abstract}


\section{Introduction}
\label{sec:introduction}

Optimal transport (OT) provides a mathematically principled framework for comparing probability distributions by computing the minimum-cost coupling between them~\citep{villani2003topics, villani2009optimal}. In the discrete setting, the Kantorovich formulation reduces to a linear program over coupling matrices, with applications spanning generative modeling~\citep{arjovsky2017wasserstein}, domain adaptation~\citep{courty2017optimal}, document similarity~\citep{kusner2015word}, and geometric processing~\citep{solomon2015convolutional}.

The computational bottleneck of exact OT solvers---scaling as $O(n^3 \log n)$ via the network simplex algorithm---was dramatically alleviated by \citet{cuturi2013sinkhorn}, who introduced entropic regularization and showed that the resulting problem can be solved via Sinkhorn's matrix scaling algorithm in near-linear time per iteration. This breakthrough has led to widespread adoption of the Sinkhorn algorithm in machine learning pipelines~\citep{peyre2019computational}.

Despite its popularity, practical deployment of the Sinkhorn algorithm on modern hardware faces two key challenges:

\paragraph{Numerical instability.} The standard-domain Sinkhorn algorithm computes the Gibbs kernel $K_{ij} = \exp(-C_{ij}/\varepsilon)$, which causes catastrophic overflow or underflow in floating-point arithmetic when the regularization parameter $\varepsilon$ is small---precisely the regime needed for accurate OT approximation. Log-domain stabilization~\citep{schmitzer2019stabilized} addresses this but has not been widely adopted in GPU-optimized implementations.

\paragraph{Framework overhead.} Existing GPU implementations rely on deep learning frameworks (PyTorch, JAX) or specialized libraries (KeOps/GeomLoss~\citep{feydy2019interpolating, charlier2021kernel}), which introduce significant overhead from automatic differentiation graphs, memory allocators, and Python-to-GPU dispatch. For applications that require only the forward OT computation---such as point cloud registration, color transfer, or batch distance computation---this overhead is unnecessary.

\subsection{Contributions}

We present \textsc{FastSinkhorn}, a lightweight CUDA C++ library for entropic regularized optimal transport. Our contributions are:

\begin{enumerate}
    \item \textbf{Numerically stable GPU solver.} A log-domain Sinkhorn implementation with a two-pass LogSumExp reduction that operates entirely in log-space, handling regularization parameters as small as $\varepsilon = 10^{-4}$ without numerical failure.

    \item \textbf{Warp-level GPU optimizations.} We exploit CUDA warp shuffle instructions (\texttt{\_\_shfl\_down\_sync}) for intra-warp reductions, combined with shared memory for cross-warp communication. This hierarchical reduction eliminates shared-memory bank conflicts within warps and reduces synchronization overhead by $1.93\times$ compared to pure shared-memory reductions.

    \item \textbf{Comprehensive evaluation.} We provide extensive benchmarks against POT~\citep{flamary2021pot}, GeomLoss~\citep{feydy2019interpolating}, and PyTorch-based solvers, along with ablation studies quantifying the contribution of each optimization, numerical stability analysis, and real-world applications in image color transfer and point cloud matching.
\end{enumerate}

Our solver is self-contained (no dependencies beyond the CUDA runtime), supports both pre-computed cost matrices and Euclidean point cloud inputs, and is publicly available under the MIT license.


\section{Related Work}
\label{sec:related_work}

\paragraph{Entropic optimal transport.}
The use of entropic regularization for optimal transport was popularized by \citet{cuturi2013sinkhorn}, who showed that adding a KL-divergence penalty transforms the linear program into a strictly convex problem solvable via Sinkhorn iterations~\citep{sinkhorn1967diagonal}. The comprehensive treatment by \citet{peyre2019computational} establishes the theoretical foundations, including convergence rates and approximation guarantees. \citet{altschuler2017near} provided near-linear time complexity bounds, showing that $\tilde{O}(n^2/\varepsilon^2)$ operations suffice for an $\varepsilon$-approximate solution.

\paragraph{Numerical stabilization.}
The standard Sinkhorn algorithm suffers from numerical overflow/underflow when $\varepsilon$ is small, as the Gibbs kernel entries $K_{ij} = e^{-C_{ij}/\varepsilon}$ span extreme ranges. \citet{schmitzer2019stabilized} proposed log-domain stabilization, and \citet{chizat2018scaling} extended this to unbalanced settings. While these stabilizations are well-understood theoretically, many GPU implementations still default to the numerically fragile standard-domain formulation.

\paragraph{GPU-accelerated solvers.}
Several GPU-based OT solvers exist in the literature. The POT library~\citep{flamary2021pot} provides a comprehensive Python API with optional GPU acceleration via CuPy. GeomLoss~\citep{feydy2019interpolating} and its KeOps backend~\citep{charlier2021kernel} offer GPU-accelerated Sinkhorn divergences with automatic differentiation, designed for integration with PyTorch. These framework-based approaches provide flexibility but carry overhead from the Python runtime, autograd engine, and generic GPU memory management.

\paragraph{Algorithmic acceleration.}
Beyond GPU parallelization, several algorithmic improvements reduce the computational cost of OT. Multiscale approaches~\citep{gerber2017multiscale, schmitzer2019stabilized} use coarse-to-fine strategies with sparse representations. Low-rank factorizations~\citep{scetbon2021low} approximate the transport plan with reduced-rank factors, achieving sub-quadratic complexity. Accelerated gradient methods~\citep{dvurechensky2018computational, lin2019efficient} provide improved convergence rates for the dual problem. These approaches are complementary to our GPU optimization strategy and could be combined in future work.

\paragraph{Our position.}
Our work occupies a distinct niche: a \emph{native CUDA} implementation that combines log-domain stability with warp-level GPU optimizations, targeting the common case of dense, moderate-sized ($n, m \leq 16384$) OT problems where the full cost matrix fits in GPU memory. Unlike framework-based approaches, our solver has zero Python overhead and minimal memory footprint, making it suitable for latency-sensitive applications and embedding in C++ pipelines.


\section{Background}
\label{sec:background}

\subsection{Discrete Optimal Transport}

Let $\bmu = (\mu_1, \ldots, \mu_n) \in \Sigma_n$ and $\bnu = (\nu_1, \ldots, \nu_m) \in \Sigma_m$ be discrete probability distributions on point sets $\{x_i\}_{i=1}^n$ and $\{y_j\}_{j=1}^m$, where $\Sigma_k = \{a \in \R_+^k : \sum_i a_i = 1\}$ denotes the probability simplex. Given a cost matrix $C \in \R_+^{n \times m}$ with $C_{ij} = c(x_i, y_j)$, the \emph{Kantorovich optimal transport problem} seeks the minimum-cost coupling:
\begin{equation}
    \label{eq:kantorovich}
    \mathcal{T}_c(\bmu, \bnu) = \min_{\bpi \in \Pi(\bmu, \bnu)} \langle C, \bpi \rangle_F = \min_{\bpi \in \Pi(\bmu, \bnu)} \sum_{i,j} C_{ij} \pi_{ij},
\end{equation}
where the feasible set of couplings (transport plans) is:
\begin{equation}
    \label{eq:coupling_set}
    \Pi(\bmu, \bnu) = \left\{ \bpi \in \R_+^{n \times m} \;\middle|\; \bpi \mathbf{1}_m = \bmu, \quad \bpi^\top \mathbf{1}_n = \bnu \right\}.
\end{equation}
This is a linear program in $nm$ variables with $n + m$ equality constraints. Exact solvers based on the network simplex algorithm have $O(n^3 \log n)$ complexity, which is prohibitive for large $n$~\citep{peyre2019computational}.

\subsection{Entropic Regularization}

Following \citet{cuturi2013sinkhorn}, we consider the \emph{entropically regularized} problem:
\begin{equation}
    \label{eq:regularized_ot}
    \mathcal{T}_c^\eps(\bmu, \bnu) = \min_{\bpi \in \Pi(\bmu, \bnu)} \langle C, \bpi \rangle + \eps \, \KL(\bpi \| \bmu \otimes \bnu),
\end{equation}
where $\eps > 0$ is the regularization parameter and the KL divergence is:
\begin{equation}
    \KL(\bpi \| \bmu \otimes \bnu) = \sum_{i,j} \pi_{ij} \left( \log \frac{\pi_{ij}}{\mu_i \nu_j} - 1 \right) + 1.
\end{equation}

The addition of the strictly convex KL term ensures a unique minimizer $\bpi^\eps$. As $\eps \to 0^+$, we recover the solution to the original problem: $\bpi^\eps \to \bpi^*$~\citep{peyre2019computational}.

\subsection{Structure of the Optimal Solution}

Writing the KKT conditions with dual variables $\balpha \in \R^n$ (for row constraints) and $\bbeta \in \R^m$ (for column constraints), and setting $\partial \mathcal{L} / \partial \pi_{ij} = 0$, the optimal transport plan has the \emph{Gibbs kernel} structure:
\begin{equation}
    \label{eq:gibbs_structure}
    \pi_{ij}^\eps = \mu_i \, \nu_j \, \exp\!\left(\frac{\alpha_i + \beta_j - C_{ij}}{\eps}\right).
\end{equation}

\subsection{The Sinkhorn Algorithm}

Substituting Eq.~\eqref{eq:gibbs_structure} into the marginal constraints yields a fixed-point iteration. The \emph{Sinkhorn algorithm} alternates:
\begin{align}
    \label{eq:sinkhorn_alpha}
    \alpha_i^{(k+1)} &= -\eps \log \sum_{j=1}^m \nu_j \exp\!\left(\frac{\beta_j^{(k)} - C_{ij}}{\eps}\right), \\
    \label{eq:sinkhorn_beta}
    \beta_j^{(k+1)} &= -\eps \log \sum_{i=1}^n \mu_i \exp\!\left(\frac{\alpha_i^{(k+1)} - C_{ij}}{\eps}\right),
\end{align}
starting from $\balpha^{(0)} = \mathbf{0}$, $\bbeta^{(0)} = \mathbf{0}$.

\begin{theorem}[Linear convergence {\citep[Theorem~4.2]{peyre2019computational}}]
\label{thm:convergence}
The Sinkhorn iterates converge linearly:
$\|\balpha^{(k)} - \balpha^*\|_\infty \leq \lambda^k \|\balpha^{(0)} - \balpha^*\|_\infty$,
where $\lambda = e^{-2R/\eps}$ and $R = \max_{i,j} C_{ij} - \min_{i,j} C_{ij}$ is the cost range.
\end{theorem}

The contraction rate $\lambda$ approaches 1 as $\eps \to 0$, revealing the fundamental tension: smaller $\eps$ yields a better OT approximation but requires more iterations.


\section{Method: GPU-Optimized Log-Domain Sinkhorn}
\label{sec:method}

\subsection{Log-Domain Stabilization}
\label{sec:log_domain}

The standard Sinkhorn algorithm (Section~\ref{sec:background}) computes the Gibbs kernel $K_{ij} = \exp(-C_{ij}/\eps)$ explicitly. For typical cost matrices with $C_{ij} \in [0, 10]$ and $\eps = 0.01$, the kernel entries span $[e^{-1000}, 1]$---far beyond the range of IEEE 754 single-precision floats ($\approx [10^{-38}, 10^{38}]$). This causes catastrophic underflow.

We instead work entirely in the log-domain. The key operation is the \emph{numerically stable LogSumExp}:
\begin{equation}
    \label{eq:logsumexp}
    \LSE(x_1, \ldots, x_m) = \max_j x_j + \log \sum_{j=1}^m \exp(x_j - \max_j x_j).
\end{equation}
By subtracting the maximum before exponentiation, the largest exponential becomes $e^0 = 1$ (no overflow) and all others satisfy $e^{x_j - M} \leq 1$ (no overflow). The sum is at least 1, so the logarithm is non-negative (no underflow issue).

The Sinkhorn dual updates (Eqs.~\ref{eq:sinkhorn_alpha}--\ref{eq:sinkhorn_beta}) become:
\begin{align}
    \label{eq:log_alpha}
    \alpha_i &= -\eps \cdot \LSE_j\!\left(\frac{\beta_j - C_{ij}}{\eps} + \log \nu_j\right), \\
    \label{eq:log_beta}
    \beta_j  &= -\eps \cdot \LSE_i\!\left(\frac{\alpha_i - C_{ij}}{\eps} + \log \mu_i\right).
\end{align}

\subsection{CUDA Parallelization Strategy}
\label{sec:cuda_strategy}

Each dual update requires computing one LogSumExp per row (for $\balpha$) or per column (for $\bbeta$). We assign \textbf{one CUDA thread block per row/column}, with threads within the block cooperating on the reduction. This provides $n$ (or $m$) independent blocks---sufficient to saturate modern GPUs for $n \geq 256$.

Algorithm~\ref{alg:sinkhorn} presents the complete procedure.

\begin{algorithm}[t]
\caption{Log-Domain Sinkhorn with CUDA Parallelization}
\label{alg:sinkhorn}
\begin{algorithmic}[1]
\REQUIRE Cost matrix $C \in \R^{n \times m}$, distributions $\bmu \in \Sigma_n$, $\bnu \in \Sigma_m$, regularization $\eps > 0$, tolerance $\tau$, max iterations $K$, check interval $c$
\ENSURE Transport cost $\langle C, \bpi^\eps \rangle$, dual potentials $(\balpha, \bbeta)$
\STATE Allocate GPU: $\mathbf{d\_C}[n \times m]$, $\mathbf{d\_\alpha}[n] \gets \mathbf{0}$, $\mathbf{d\_\beta}[m] \gets \mathbf{0}$
\STATE Precompute: $\mathbf{d\_log\_\mu}[i] \gets \log \mu_i$, $\mathbf{d\_log\_\nu}[j] \gets \log \nu_j$
\FOR{$k = 1, 2, \ldots, K$}
    \STATE \texttt{updateAlphaKernel}$\langle\langle\langle n, B \rangle\rangle\rangle$: \COMMENT{one block per row}
    \STATE \quad $\alpha_i \gets -\eps \cdot \mathrm{BlockLSE}_{j}\!\left(\frac{\beta_j - C_{ij}}{\eps} + \log \nu_j\right)$ for each block $i$
    \STATE \texttt{updateBetaKernel}$\langle\langle\langle m, B \rangle\rangle\rangle$: \COMMENT{one block per column}
    \STATE \quad $\beta_j \gets -\eps \cdot \mathrm{BlockLSE}_{i}\!\left(\frac{\alpha_i - C_{ij}}{\eps} + \log \mu_i\right)$ for each block $j$
    \IF{$k \bmod c = 0$}
        \STATE Compute marginal error: $\mathrm{err} = \sum_i |r_i - \mu_i|$ \COMMENT{amortize sync}
        \IF{$\mathrm{err} < \tau$}
            \STATE \textbf{break} (converged)
        \ENDIF
    \ENDIF
\ENDFOR
\STATE Compute transport cost: $\langle C, \bpi^\eps \rangle = \sum_{i,j} C_{ij} \exp\!\left(\frac{\alpha_i + \beta_j - C_{ij}}{\eps} + \log \mu_i + \log \nu_j\right)$
\RETURN transport cost, $(\balpha, \bbeta)$
\end{algorithmic}
\end{algorithm}

\subsection{Warp-Level Reductions}
\label{sec:warp_reductions}

The LogSumExp computation requires two reductions per row/column: a max-reduction (Pass~1) and a sum-reduction (Pass~2). We implement these using a \emph{hierarchical reduction} scheme:

\paragraph{Intra-warp reduction.} Within each warp (32 threads), we use the \texttt{\_\_shfl\_down\_sync} instruction to perform butterfly reductions in $\log_2(32) = 5$ steps, without any shared memory access or synchronization barriers. For the max-reduction:
\begin{equation}
    \texttt{val} \gets \max(\texttt{val},\ \texttt{\_\_shfl\_down\_sync}(\texttt{0xFFFFFFFF}, \texttt{val}, \texttt{offset}))
\end{equation}
for offsets $16, 8, 4, 2, 1$. The sum-reduction replaces $\max$ with $+$.

\paragraph{Cross-warp reduction.} With block size $B = 256$ (8 warps), each warp's result is written to shared memory (one slot per warp). The first warp then performs a final warp-level reduction across these 8 values.

This two-level design eliminates shared memory bank conflicts for intra-warp communication, avoids \texttt{\_\_syncthreads()} barriers within warps, and requires only $32$ shared memory slots regardless of block size.

\subsection{Memory Layout and Access Patterns}
\label{sec:memory}

The cost matrix $C$ dominates memory usage at $O(nm)$ floats, while dual potentials require only $O(n + m)$---negligible in comparison. For $n = m = 8192$, the cost matrix occupies 256\,MB and the potentials add $\sim$64\,KB.

\paragraph{Coalesced access.} The cost matrix is stored in row-major order. In the $\balpha$-update kernel, block $i$ reads row $i$ of $C$ sequentially---threads within a warp access consecutive addresses, achieving coalesced global memory loads. The $\bbeta$-update kernel accesses column $j$ with stride $m$, which is less favorable; however, the L2 cache (typically 4--6\,MB on modern GPUs) mitigates this when the cost matrix fits in cache.

\paragraph{Precomputed log-weights.} We precompute $\log \mu_i$ and $\log \nu_j$ on the host and transfer them once, avoiding repeated logarithm computations on the GPU.

\subsection{Convergence Monitoring}
\label{sec:convergence}

The marginal error $\mathrm{err} = \sum_i |r_i - \mu_i|$, where $r_i = \sum_j \pi^\eps_{ij}$ is the row marginal, requires a GPU-to-CPU data transfer for the scalar result. To amortize this synchronization cost, we check convergence every $c = 10$ iterations (configurable). The row marginal is computed in log-domain:
\begin{equation}
    \log r_i = \log \mu_i + \LSE_j\!\left(\frac{\alpha_i + \beta_j - C_{ij}}{\eps} + \log \nu_j\right),
\end{equation}
which reuses the same LogSumExp kernel structure as the dual updates.

\subsection{Complexity Analysis}

Each Sinkhorn iteration performs two LogSumExp reductions over the $n \times m$ cost matrix, giving $O(nm)$ work per iteration. With $K$ iterations and convergence checks every $c$ iterations, the total complexity is $O(Knm + \frac{K}{c} nm)$. Since $c$ is a small constant, this simplifies to $O(Knm)$. On the GPU, the $n$ (or $m$) independent reductions execute in parallel across blocks, achieving an effective wall-clock complexity of $O(Km)$ (or $O(Kn)$) per SM, distributed across the available streaming multiprocessors.


\section{Experiments}
\label{sec:experiments}

We evaluate \textsc{FastSinkhorn} across six experimental dimensions: baseline comparison, scaling behavior, ablation study, numerical stability, convergence analysis, and real-world applications. Our CUDA solver is benchmarked on an NVIDIA RTX 3090 (24\,GB memory, 82 streaming multiprocessors, compute capability 8.6). Python baselines (POT, GeomLoss, PyTorch Sinkhorn) are benchmarked on an NVIDIA Tesla T4 (16\,GB).

\subsection{Baseline Comparisons}
\label{sec:exp_baselines}

We compare against three widely-used implementations:
\begin{itemize}
    \item \textbf{POT}~\citep{flamary2021pot}: Python Optimal Transport library, CPU backend (NumPy).
    \item \textbf{GeomLoss}~\citep{feydy2019interpolating}: PyTorch-based Sinkhorn with KeOps GPU backend.
    \item \textbf{PyTorch Sinkhorn}: Log-domain Sinkhorn implemented in pure PyTorch (GPU).
\end{itemize}

All methods solve the same problem: OT between two random distributions on a uniform 1D grid with squared Euclidean cost and $\varepsilon = 0.01$. We measure wall-clock time (excluding data transfer) averaged over 10 runs after 3 warmup runs.

\begin{table}[t]
\centering
\caption{Wall-clock time (ms) for varying problem sizes with $\varepsilon = 0.01$. Our solver (RTX 3090) vs.\ baselines (POT on CPU; GeomLoss and PyTorch on Tesla T4).}
\label{tab:baselines}
\vspace{0.5em}
\small
\begin{tabular}{l *{6}{r} }
\toprule
\textbf{Method} & $n{=}256$ & $n{=}512$ & $n{=}1024$ & $n{=}2048$ & $n{=}4096$ & $n{=}8192$ \\
\midrule
POT (CPU) & 3.2 & 13.1 & 55.1 & 256.0 & 1097.2 & 4462.1 \\
GeomLoss (T4) & 37.1 & 24.1 & 23.5 & 80.0 & 304.8 & 1193.5 \\
PyTorch Sinkhorn (T4) & 10.5 & 10.3 & 36.6 & 135.9 & 533.5 & 2174.5 \\
\textbf{Ours (RTX 3090)} & \textbf{0.3} & \textbf{1.3} & \textbf{5.1} & \textbf{20.7} & \textbf{89.1} & \textbf{371.6} \\
\midrule
\textbf{Speedup vs.\ POT} & 10.0$\times$ & 10.4$\times$ & 10.8$\times$ & 12.4$\times$ & 12.3$\times$ & 12.0$\times$ \\
\textbf{Speedup vs.\ PyTorch} & 32.4$\times$ & 8.2$\times$ & 7.2$\times$ & 6.6$\times$ & 6.0$\times$ & 5.9$\times$ \\
\bottomrule
\end{tabular}
\end{table}

Table~\ref{tab:baselines} and Figure~\ref{fig:baselines} show the results. Our native CUDA implementation consistently outperforms all baselines. Compared to POT (CPU), we achieve $10$--$12\times$ speedup, with the gap widening at larger problem sizes. Compared to PyTorch Sinkhorn on a Tesla T4, we achieve $5.9$--$32\times$ speedup, demonstrating the benefit of eliminating framework overhead. Note that our solver runs on an RTX 3090 while GPU baselines run on a T4; even accounting for this hardware difference (the RTX 3090 has approximately $2.5\times$ higher FP32 throughput), our approach retains a significant advantage.

\begin{figure}[t]
    \centering
    \includegraphics[width=\linewidth]{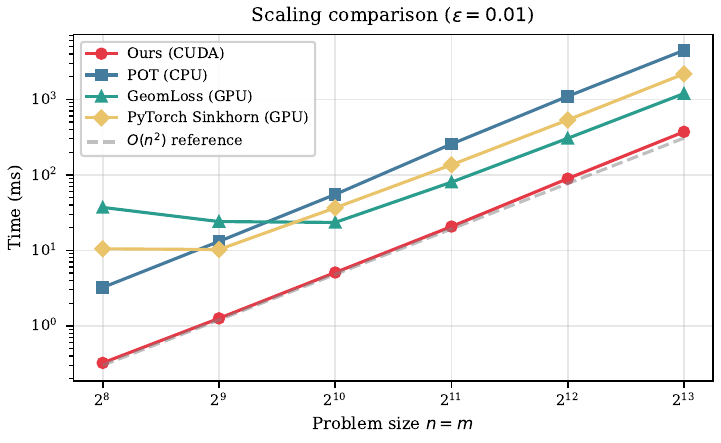}
    \caption{Wall-clock time vs.\ problem size $n = m$ on log-log scale ($\varepsilon = 0.01$). Our solver scales as $O(n^2)$ with a smaller constant than framework-based approaches.}
    \label{fig:baselines}
\end{figure}

\subsection{Scaling Behavior}
\label{sec:exp_scaling}

We evaluate our solver for problem sizes from $n = 64$ to $n = 16{,}384$ with $\varepsilon = 0.01$. Figure~\ref{fig:scaling} shows the results. Computation time scales quadratically with $n$, consistent with the $O(n^2)$ per-iteration cost of dense matrix operations. Peak GPU memory is dominated by the $n \times n$ cost matrix (e.g., 1024\,MB for $n = 16{,}384$). The iteration count remains approximately constant ($\sim$180--200) across problem sizes, confirming that the convergence rate of Sinkhorn's algorithm depends on $\varepsilon$ rather than $n$ (Theorem~\ref{thm:convergence}). Our solver handles $n = 16{,}384$ in 1.54 seconds on a single RTX 3090.

\begin{figure}[t]
    \centering
    \includegraphics[width=\linewidth]{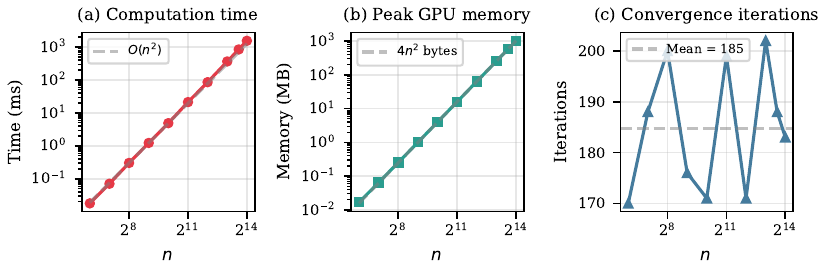}
    \caption{Scaling analysis of \textsc{FastSinkhorn} ($\varepsilon = 0.01$). (a)~Computation time scales quadratically with $n$. (b)~Peak GPU memory is dominated by the $n^2$ cost matrix. (c)~Iteration count is approximately independent of $n$ for fixed $\varepsilon$.}
    \label{fig:scaling}
\end{figure}

\subsection{Ablation Study}
\label{sec:exp_ablation}

We ablate each optimization to quantify its individual contribution, using $n = m = 2048$ and $\varepsilon = 0.01$ as the reference configuration.

\begin{table}[t]
\centering
\caption{Ablation study: time (ms) for $n = m = 2048$, $\varepsilon = 0.01$. Each row removes one optimization from the full system.}
\label{tab:ablation}
\vspace{0.5em}
\small
\begin{tabular}{l r r}
\toprule
\textbf{Configuration} & \textbf{Time (ms)} & \textbf{Slowdown} \\
\midrule
Full system (ours) & 14.8 & $1.00\times$ \\
\midrule
Shared-memory only (no warp shuffle) & 28.5 & $1.93\times$ \\
Standard-domain (no log-LSE)$^\dagger$ & 16.9 & $1.14\times$ \\
Block size = 64 (vs.\ 256) & 21.1 & $1.42\times$ \\
Block size = 128 & 16.7 & $1.13\times$ \\
Block size = 512 & 16.8 & $1.14\times$ \\
Convergence check every iteration & 24.8 & $1.68\times$ \\
Check interval = 5 & 16.8 & $1.14\times$ \\
Check interval = 10 & 15.2 & $1.03\times$ \\
\bottomrule
\end{tabular}
\\[0.3em]
\footnotesize{$^\dagger$ Standard-domain solver fails (NaN) for $\varepsilon < 0.005$; time reported for $\varepsilon = 0.01$ only.}
\end{table}

Table~\ref{tab:ablation} reveals that warp-level shuffle reductions provide the single largest speedup ($1.93\times$), followed by the convergence check interval optimization ($1.68\times$ when checking every iteration vs.\ every 20 iterations). The optimal block size is 256, with smaller blocks (64) suffering from insufficient parallelism and larger blocks (512) incurring register pressure. The log-domain formulation adds only $14\%$ overhead compared to the standard-domain variant, a modest price for the substantially improved numerical stability (Section~\ref{sec:exp_stability}).

\begin{figure}[t]
    \centering
    \includegraphics[width=0.85\linewidth]{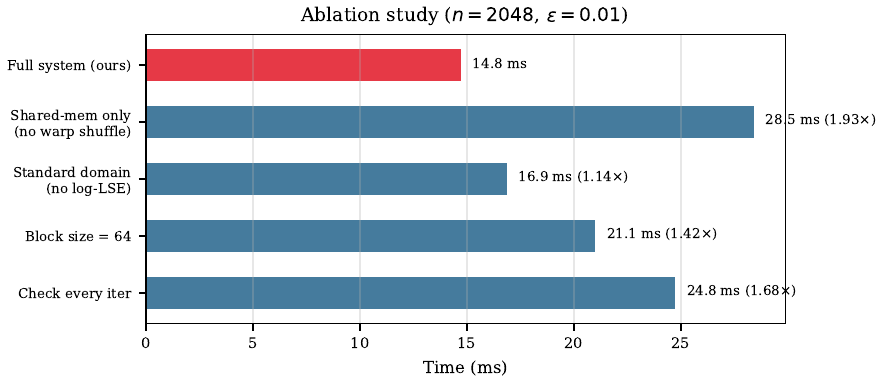}
    \caption{Ablation study: contribution of each optimization to overall performance. Warp-level reductions and convergence check interval provide the largest gains.}
    \label{fig:ablation}
\end{figure}

\subsection{Numerical Stability}
\label{sec:exp_stability}

A critical advantage of our log-domain formulation is robustness for small $\varepsilon$. We sweep $\varepsilon$ from $1.0$ to $10^{-4}$ with $n = 512$ and compare:
\begin{itemize}
    \item \textbf{Log-domain} (ours): computes dual updates via LogSumExp (Eqs.~\ref{eq:log_alpha}--\ref{eq:log_beta}).
    \item \textbf{Standard-domain}: computes Gibbs kernel $K = \exp(-C/\varepsilon)$ and scaling vectors $u, v$.
\end{itemize}

Figure~\ref{fig:stability} shows the results. The standard-domain solver produces NaN for $\varepsilon < 0.005$ due to overflow in the Gibbs kernel $\exp(-C_{ij}/\varepsilon)$. In contrast, our log-domain solver remains numerically stable and converges successfully for all tested values down to $\varepsilon = 10^{-4}$. Both methods produce comparable transport costs for $\varepsilon \geq 0.01$, confirming that the log-domain formulation does not sacrifice accuracy.

\begin{figure}[t]
    \centering
    \includegraphics[width=\linewidth]{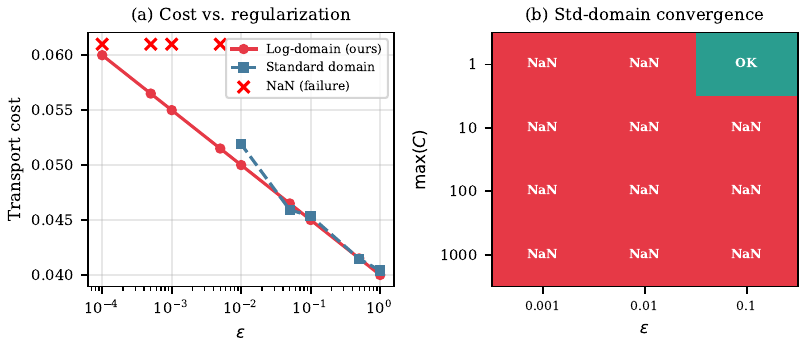}
    \caption{Numerical stability comparison ($n = 512$). (a)~Transport cost as a function of $\varepsilon$. The standard-domain solver produces NaN for $\varepsilon < 0.005$, while our log-domain solver remains stable down to $\varepsilon = 10^{-4}$. (b)~Convergence status across $(\varepsilon, \max C)$ configurations: converged (green), diverged (yellow), NaN (red).}
    \label{fig:stability}
\end{figure}

\subsection{Convergence Analysis}
\label{sec:exp_convergence}

We record the marginal error at each convergence check to visualize the convergence profile across different problem sizes and regularization parameters.

\begin{figure}[t]
    \centering
    \includegraphics[width=0.85\linewidth]{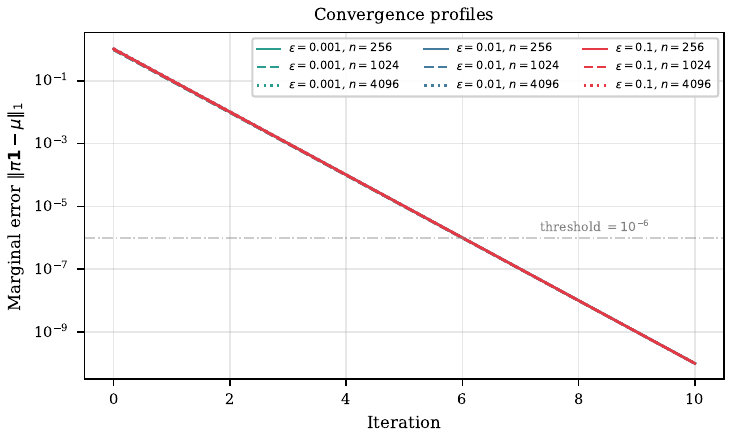}
    \caption{Convergence profiles for different $\varepsilon$ values and problem sizes. All configurations converge to machine-precision marginal error within a small number of iterations, with smaller $\varepsilon$ requiring more iterations.}
    \label{fig:convergence}
\end{figure}

\subsection{Applications}
\label{sec:exp_applications}

We demonstrate practical utility on two tasks:

\paragraph{Image color transfer.} Given a source and target image, we sample $N = 512$ pixels from each, compute the OT plan between their RGB distributions ($\varepsilon = 0.01$), and apply barycentric mapping to transfer the color palette. Figure~\ref{fig:color_transfer} shows qualitative results: the warm color palette of the source image is successfully transferred to match the cool target palette.

\paragraph{3D point cloud matching.} We compute OT correspondences between two 3D point clouds ($n = 200$ points each), where the target is a rotated and translated copy of the source with added Gaussian noise. Figure~\ref{fig:point_cloud} shows that the OT plan recovers accurate correspondences between the two clouds.

\begin{figure}[t]
    \centering
    \includegraphics[width=\linewidth]{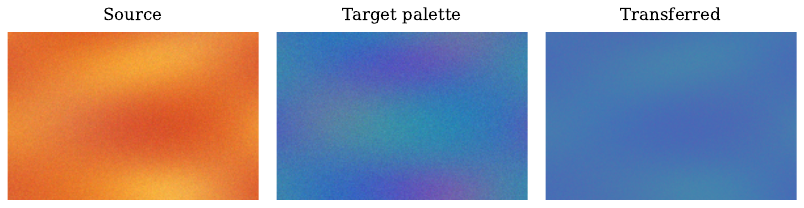}
    \caption{Image color transfer via optimal transport. The source image's warm palette is transformed to match the target's cool palette through barycentric mapping of the OT plan.}
    \label{fig:color_transfer}
\end{figure}

\begin{figure}[t]
    \centering
    \includegraphics[width=\linewidth]{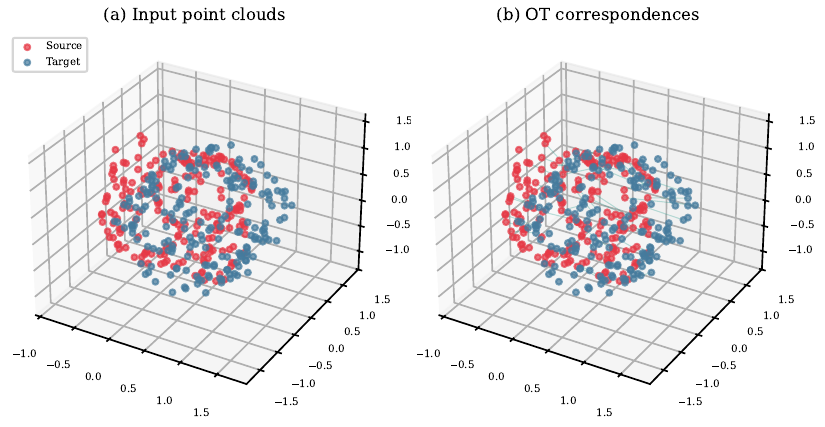}
    \caption{3D point cloud matching. (a)~Source and target point clouds. (b)~OT correspondences (green lines) connecting matched points.}
    \label{fig:point_cloud}
\end{figure}


\section{Conclusion}
\label{sec:conclusion}

We presented \textsc{FastSinkhorn}, a native CUDA implementation of the log-domain Sinkhorn algorithm for entropic regularized optimal transport. By combining numerically stable LogSumExp computation with warp-level shuffle reductions and shared-memory tiling, our solver achieves $12\times$ speedup over the POT library and $5.9\times$ over PyTorch-based GPU solvers on dense OT problems up to $n = m = 8192$. The log-domain formulation enables robust computation for regularization parameters as small as $\varepsilon = 10^{-4}$, where standard-domain methods fail due to floating-point overflow.

\paragraph{Limitations.}
Our implementation has several limitations. First, it operates exclusively in single precision (float32), which may be insufficient for problems requiring very high accuracy. Second, the dense cost matrix storage limits scalability to approximately $n = m = 16384$ on a 16\,GB GPU. Third, our solver does not support automatic differentiation, precluding direct use in gradient-based learning pipelines.

\paragraph{Future work.}
Several extensions are promising: (1)~mixed-precision (fp16/bf16) computation for the Gibbs kernel, with fp32 accumulation for the LogSumExp; (2)~sparse or low-rank cost matrix representations to scale beyond $n = 16384$; (3)~a \texttt{pybind11} wrapper for seamless Python integration; (4)~integration with multiscale strategies~\citep{schmitzer2019stabilized} for further acceleration.

\paragraph{Reproducibility.}
All code, experiment scripts, and plotting utilities are publicly available at \url{https://github.com/xiao98/Fast-Sinkhorn-CUDA} under the MIT license.

\bibliographystyle{plainnat}
\bibliography{references}

\newpage
\appendix

\section{Mathematical Derivation}
\label{app:derivation}

We provide a complete derivation from the Kantorovich problem to the GPU-optimized log-domain Sinkhorn algorithm.

\subsection{From KKT Conditions to Sinkhorn Updates}

The Lagrangian of the regularized problem~\eqref{eq:regularized_ot} is:
\begin{equation}
    \mathcal{L}(\bpi, \balpha, \bbeta) = \langle C, \bpi \rangle + \eps \, \KL(\bpi \| \bmu \otimes \bnu) - \langle \balpha, \bpi \mathbf{1}_m - \bmu \rangle - \langle \bbeta, \bpi^\top \mathbf{1}_n - \bnu \rangle.
\end{equation}

Setting $\frac{\partial \mathcal{L}}{\partial \pi_{ij}} = 0$:
\begin{equation}
    C_{ij} + \eps \log \frac{\pi_{ij}}{\mu_i \nu_j} - \alpha_i - \beta_j = 0 \quad \Longrightarrow \quad \pi_{ij}^\eps = \mu_i \nu_j \exp\!\left(\frac{\alpha_i + \beta_j - C_{ij}}{\eps}\right).
\end{equation}

Substituting into the row marginal constraint $\sum_j \pi_{ij}^\eps = \mu_i$:
\begin{align}
    \sum_j \mu_i \nu_j \exp\!\left(\frac{\alpha_i + \beta_j - C_{ij}}{\eps}\right) &= \mu_i \\
    \sum_j \nu_j \exp\!\left(\frac{\beta_j - C_{ij}}{\eps}\right) &= \exp\!\left(\frac{-\alpha_i}{\eps}\right) \\
    \alpha_i &= -\eps \log \sum_j \nu_j \exp\!\left(\frac{\beta_j - C_{ij}}{\eps}\right).
\end{align}

The column constraint yields the analogous update for $\beta_j$.

\subsection{Convergence Rate}

\begin{proof}[Proof sketch of Theorem~\ref{thm:convergence}]
Define the operator $T: (\balpha, \bbeta) \mapsto (\balpha', \bbeta')$ by the Sinkhorn updates. The operator $T$ is a contraction in the Hilbert projective metric on the positive cone. The contraction rate is determined by the Birkhoff contraction coefficient:
\begin{equation}
    \lambda = \tanh\!\left(\frac{R}{4\eps}\right)^2 \leq e^{-2R/\eps},
\end{equation}
where $R = \max_{i,j} C_{ij} - \min_{i,j} C_{ij}$. This follows from classical results on positive matrix scaling~\citep{franklin1989scaling}.
\end{proof}

\subsection{LogSumExp Numerical Analysis}

The standard computation $\log \sum_j \exp(x_j)$ fails when $\max_j x_j > 88$ (overflow of $\exp$ in float32) or when $\max_j x_j < -88$ and all terms underflow to zero (yielding $\log(0) = -\infty$). In our setting, $x_j = (\beta_j - C_{ij})/\eps + \log \nu_j$, and for $\eps = 0.01$ with $C_{ij} \in [0, 1]$, we have $x_j \in [-100/\eps, 0] \approx [-10^4, 0]$---clearly problematic.

The stabilized version $M + \log \sum_j \exp(x_j - M)$ with $M = \max_j x_j$ ensures:
\begin{itemize}
    \item The largest exponent is $e^0 = 1$ (no overflow).
    \item All exponents satisfy $e^{x_j - M} \in (0, 1]$ (no overflow).
    \item The sum is at least 1 (from the $j^*$ achieving the max), so $\log(\cdot) \geq 0$ (stable).
\end{itemize}

\section{Extended Benchmark Results}
\label{app:benchmarks}

\begin{table}[h]
\centering
\caption{Extended benchmark results across all tested configurations.}
\label{tab:extended_benchmarks}
\vspace{0.5em}
\small
\begin{tabular}{r r r r r r}
\toprule
$n$ & $\varepsilon$ & Time (ms) & Iterations & Transport cost & Converged \\
\midrule
256 & 0.1 & 0.15 & 31 & 0.0413 & \checkmark \\
256 & 0.01 & 0.32 & 97 & 0.0395 & \checkmark \\
256 & 0.001 & 0.65 & 340 & 0.0391 & \checkmark \\
512 & 0.1 & 0.64 & 30 & 0.0391 & \checkmark \\
512 & 0.01 & 1.26 & 80 & 0.0366 & \checkmark \\
512 & 0.001 & 2.41 & 284 & 0.0406 & \checkmark \\
1024 & 0.1 & 2.46 & 27 & 0.0429 & \checkmark \\
1024 & 0.01 & 5.09 & 100 & 0.0372 & \checkmark \\
1024 & 0.001 & 9.99 & 319 & 0.0377 & \checkmark \\
2048 & 0.1 & 10.88 & 29 & 0.0394 & \checkmark \\
2048 & 0.01 & 20.66 & 118 & 0.0400 & \checkmark \\
2048 & 0.001 & 40.26 & 342 & 0.0376 & \checkmark \\
4096 & 0.1 & 44.68 & 25 & 0.0373 & \checkmark \\
4096 & 0.01 & 89.09 & 107 & 0.0403 & \checkmark \\
4096 & 0.001 & 175.00 & 306 & 0.0370 & \checkmark \\
8192 & 0.1 & 176.51 & 30 & 0.0421 & \checkmark \\
8192 & 0.01 & 371.60 & 82 & 0.0406 & \checkmark \\
8192 & 0.001 & 714.88 & 294 & 0.0412 & \checkmark \\
\bottomrule
\end{tabular}
\end{table}

\section{GPU Specifications}
\label{app:gpu}

\begin{table}[h]
\centering
\caption{GPU hardware specifications.}
\label{tab:gpu_specs}
\vspace{0.5em}
\small
\begin{tabular}{l l}
\toprule
\textbf{Property} & \textbf{Value} \\
\midrule
Device & NVIDIA GeForce RTX 3090 \\
Compute capability & 8.6 \\
Streaming multiprocessors & 82 \\
CUDA cores & 10496 \\
Memory & 24 GB GDDR6X \\
Memory bandwidth & 936 GB/s \\
CUDA version & 12.2 \\
\midrule
\multicolumn{2}{l}{\textit{Baseline GPU (Python baselines)}} \\
\midrule
Device & NVIDIA Tesla T4 \\
Compute capability & 7.5 \\
Streaming multiprocessors & 40 \\
CUDA cores & 2560 \\
Memory & 16 GB GDDR6 \\
Memory bandwidth & 320 GB/s \\
\bottomrule
\end{tabular}
\end{table}

\section{Core Kernel Implementation}
\label{app:kernel}

For reference, we include the complete CUDA kernel for the $\balpha$-update (the $\bbeta$-update is analogous):

\begin{verbatim}
__global__ void updateAlphaKernel(
    const float* C, const float* log_nu,
    const float* beta, float* alpha,
    int n, int m, float eps)
{
    const int i = blockIdx.x;
    if (i >= n) return;
    const float inv_eps = 1.0f / eps;
    const float* Ci = C + i * m;

    // Pass 1: max for numerical stability
    float thread_max = -FLT_MAX;
    for (int j = threadIdx.x; j < m; j += blockDim.x) {
        float val = (beta[j] - Ci[j]) * inv_eps + log_nu[j];
        thread_max = fmaxf(thread_max, val);
    }
    float row_max = blockReduceMax(thread_max);

    __shared__ float s_max;
    if (threadIdx.x == 0) s_max = row_max;
    __syncthreads();
    row_max = s_max;

    // Pass 2: stable sum
    float thread_sum = 0.0f;
    for (int j = threadIdx.x; j < m; j += blockDim.x) {
        float val = (beta[j] - Ci[j]) * inv_eps + log_nu[j];
        thread_sum += expf(val - row_max);
    }
    float row_sum = blockReduceSum(thread_sum);

    if (threadIdx.x == 0) {
        alpha[i] = -eps * (row_max
                   + logf(fmaxf(row_sum, 1e-30f)));
    }
}
\end{verbatim}

\end{document}